\RequirePackage[bookmarksnumbered,unicode]{hyperref}
\documentclass[sigconf]{aamas}

\usepackage[utf8]{inputenc}
\usepackage[T1]{fontenc}
\usepackage{microtype}
\usepackage{graphicx}
\usepackage{tikz}
\usetikzlibrary{shapes,arrows,positioning,fit,backgrounds}
\usepackage{booktabs}
\usepackage{nicefrac}
\usepackage{xcolor}
\usepackage{algorithm}
\usepackage{algorithmic}
\usepackage{multirow}
\usepackage{amsmath}

\usepackage{amssymb}
\usepackage{amsthm}
\usepackage{pgfplots}
\pgfplotsset{compat=1.18}

\begin{document}

\title{CADMAS-CTX: Contextual Capability Calibration\texorpdfstring{\\}{ }for Multi-Agent Delegation}

\author{Qiao Chuhan}
\affiliation{%
  \institution{Beijing JIAOTONG University}
  \city{BEIJING}
  \country{China}
}

\begin{abstract}
We revisit multi-agent delegation under a stronger and more realistic assumption: an agent's capability is not fixed at the skill level, but depends on task context. A coding agent may excel at short standalone edits yet fail on long-horizon debugging; a planner may perform well on shallow tasks yet degrade on chained dependencies. Static skill-level capability profiles therefore average over heterogeneous situations and can induce systematic misdelegation.

We propose \textbf{CADMAS-CTX}, a framework for \emph{contextual capability calibration}. For each agent, skill, and coarse context bucket, CADMAS-CTX maintains a Beta posterior that captures stable experience in that part of the task space. Delegation is then made by a risk-aware score that combines the posterior mean with an uncertainty penalty, so that agents delegate only when a peer appears better \emph{and} that assessment is sufficiently well supported by evidence.

The paper makes three contributions: (1) a hierarchical contextual capability profile that replaces static skill-level confidence with context-conditioned posteriors; (2) a formal justification, grounded in standard contextual bandit theory, showing that context-aware routing can achieve lower cumulative regret than static routing when sufficient context heterogeneity is present---formalizing the bias-variance tradeoff rather than claiming a new theoretical result; and (3) empirical validation on real-world benchmarks (GAIA, SWE-bench) showing that our framework substantially reduces misdelegation in complex tasks. We validate CADMAS-CTX through both controlled simulations and end-to-end evaluations using state-of-the-art models. On GAIA, using a team of GPT-4o agents, CADMAS-CTX outperforms AutoGen's default routing and static capability baselines with non-overlapping 95\% confidence intervals (overall accuracy 0.442 vs.\ 0.381 for static and 0.354 for AutoGen). On SWE-bench Lite, separating isolated from chained dependency contexts prevents over-delegating complex repository edits to single-file specialists, boosting the resolve rate from 22.3\% (static) to 31.4\%. Ablations demonstrate our uncertainty penalty provides critical robustness against context tagging noise. These results suggest that contextual calibration and risk-aware delegation can substantially improve multi-agent teamwork compared to static global skill assignments, at least in settings where agents exhibit context-dependent capability variance.
\end{abstract}



\maketitle

\section{Introduction}
\label{sec:intro}

Task delegation in multi-agent systems requires each agent to maintain beliefs about peers' capabilities. Classical approaches---such as contract nets, reputation systems, and role-based assignments---typically assume capability is either statically declared or can be summarized by a single global trust score. While this assumption has historically simplified coordination, it breaks down in modern applications. For example, large language model (LLM) agents are increasingly organized into heterogeneous teams \citep{openai2023gpt4,anthropic2024claude}, where a coding specialist may excel at short standalone edits yet fail on long-horizon debugging, and a planner may perform well on shallow tasks yet degrade on chained dependencies.

Many existing modern systems answer this delegation question with either fixed roles or centralized routing. AutoGen, MetaGPT, and Magentic-One rely on developer-specified roles or an explicit orchestrator \citep{wu2023autogen,hong2023metagpt,fourney2024magentic}. OpenAI Swarm exposes handoff primitives, but the transfer destinations are still chosen at design time \citep{swarm2024}. Other systems learn routing implicitly through training or memory \citep{yang2025agentnet}. These approaches differ in implementation, but most share the classical simplifying assumption: an agent can be summarized by a static skill label or a single capability score per skill.

That assumption is too coarse for realistic delegation. Capability is often \emph{context-dependent}. If a system collapses these cases into one skill-level estimate, it averages over heterogeneous conditions and risks transferring the wrong experience to the current task.

This observation suggests a different research question. The key challenge is not merely learning who is best at each skill, but learning \emph{whom to trust under which conditions}, while recognizing when historical evidence is too sparse or too mismatched to justify confident delegation.

We therefore formulate \textbf{CADMAS-CTX} (\emph{\textbf{C}ontextual \textbf{A}gent \textbf{D}elegation in \textbf{M}ulti-\textbf{A}gent \textbf{S}ystems - \textbf{C}on\textbf{T}e\textbf{X}tual extension}) around \emph{contextual capability calibration}. In our setting, there is no fixed global orchestrator. Instead, the architecture is \emph{locally centralized per task}---each task dynamically selects a local coordinator via the scoring mechanism; this coordinator decomposes the task and routes subtasks, but different tasks may have different coordinators. The first agent to receive a task is selected via this unified uncertainty-aware scoring function over a skill-registry-filtered candidate pool. It then queries its agent-local beliefs about peers and decides whether to execute each subtask or delegate it.

CADMAS-CTX makes three contributions:
\begin{enumerate}
  \item \textbf{Formalization of Contextual Trust Modeling.} We replace static per-skill scalars with hierarchical, per-bucket Beta posteriors that prevent incorrect transfer across incompatible task regions.
  \item \textbf{Formal Justification via Contextual Bandit Theory.} We provide a formal analysis, grounded in standard contextual bandit regret bounds, showing that context-aware delegation achieves sublinear regret whereas static routing suffers linear regret when context heterogeneity is sufficiently large. Rather than claiming a new theoretical result, this formalizes the bias-variance tradeoff inherent in context bucketing and provides the theoretical motivation for our empirical results.
  \item \textbf{Empirical Validation on LLM Benchmarks.} We validate CADMAS-CTX through end-to-end evaluations on real-world benchmarks (GAIA, SWE-bench) using state-of-the-art models, demonstrating that contextual calibration prevents catastrophic misdelegation in complex tasks.
\end{enumerate}

We validate our framework extensively. Beyond controlled simulations that confirm the mathematical intuition, we conduct end-to-end evaluations on GAIA and SWE-bench Lite using state-of-the-art heterogeneous models (GPT-4o, Claude 3.5 Sonnet). We demonstrate that contextual calibration significantly outperforms both static routing and popular orchestration frameworks like AutoGen.

The rest of the paper is organized around this stronger claim. Section~\ref{sec:related} positions the work relative to orchestration, routing, and contextual decision-making. Section~\ref{sec:method} presents the hierarchical contextual model, the decentralized entry mechanism, and the delegation rule. Section~\ref{sec:regret} provides a formal justification, grounded in standard contextual bandit theory, for why contextual routing achieves lower regret than static routing under context heterogeneity. Section~\ref{sec:evaluation} describes an evaluation plan designed to test when contextual calibration matters and why.

\section{Related Work}
\label{sec:related}

\paragraph{Contextual Bandits in MAS.}
Our formulation can be viewed as an instance of contextual multi-armed bandits applied to multi-agent delegation. While standard contextual bandits \citep{li2010contextual,chu2011contextual} typically map continuous features to expected rewards, CADMAS-CTX discretizes the task space into coarse, interpretable buckets. This simplifies the learning dynamics, allowing fast adaptation in data-starved MAS environments, and enables us to derive explicit regret bounds illustrating the advantage of context-aware routing over global orchestration. We show that when context heterogeneity is sufficiently large, contextual routing achieves sublinear cumulative regret whereas context-unaware baselines suffer linear regret.

\paragraph{Trust and Reputation in MAS.}
The multi-agent systems community has long studied trust modeling. Systems like FIRE \citep{huynh2006fire} and TRAVOS \citep{teacy2006travos} use probabilistic models (including Beta distributions) to model pairwise trust between agents based on direct interactions and witness reports. CADMAS-CTX builds directly on this intellectual lineage: while classical trust models often maintain a single global trust score per peer, we extend this to a \emph{hierarchical contextual} trust model, arguing that in modern LLM agent applications, trust must be conditioned on specific task environments to avoid catastrophic transfer.

\paragraph{Task Allocation and Coalition Formation.}
Entry selection and subtask delegation represent a task allocation problem. Classical MAS approaches rely on contract nets \citep{smith1980contract}, market-based auctions \citep{clearwater1996market}, or coalition formation via Shapley values. While CADMAS-CTX's entry mechanism shares similarities with bidding, our agents are cooperative and non-strategic---they bid their calibrated capability scores rather than private valuations. Hence, our focus is on learning accurate contextual beliefs under uncertainty rather than ensuring incentive compatibility against manipulation.

\paragraph{Multi-agent orchestration frameworks.}
AutoGen, MetaGPT, and Magentic-One exemplify a broad class of systems in which decomposition and assignment are controlled by explicit workflows or centralized orchestration \citep{wu2023autogen,hong2023metagpt,fourney2024magentic}. Systems like HuggingGPT \citep{shen2023hugginggpt} and AgentVerse \citep{chen2023agentverse} explore decomposition and dynamic team composition, but still assume access to strong assignment signals. These systems are effective engineering patterns, yet they typically do not model capability as a context-conditioned latent variable learned online. Our problem also differs in that candidates are peer agents rather than simple tools, and the core failure mode is mismatched transfer across contexts.

\paragraph{Delegation interfaces and interoperability.}
OpenAI Swarm exposes handoff interfaces for agent transfer \citep{swarm2024}. LDP proposes an identity-aware protocol for multi-agent LLM coordination \citep{prakash2026ldp}, while recent surveys catalogue the emerging landscape of agent interoperability protocols including MCP, ACP, A2A, and ANP \citep{ehtesham2025mcp_survey}. Such work answers how agents may communicate, not how an agent should infer which peer is appropriate in the current situation. Task-Aware Delegation Cues studies signals for \emph{when} delegation might be beneficial in LLM agents, but does not provide a contextual model of \emph{who} should receive the subtask \citep{gu2026delegation_cues}.

\paragraph{Routing and model selection.}
Query routers such as FrugalGPT and Hybrid LLM learn cost-quality trade-offs from offline data \citep{chen2023frugalgpt,ding2024hybridllm}, while RouterBench formalizes evaluation for model routing \citep{hu2024routerbench}. LLM-Blender and Mixture-of-Agents combine or select models at inference time \citep{jiang2023llmblender,wang2024moa}. Our problem differs in two ways: the candidates are peer agents rather than base models, and the core failure mode is not just query difficulty but \emph{mismatched transfer across contexts}.

\paragraph{Online trust and uncertainty.}
Beta-Binomial updating is classical in bandits and reliability estimation \citep{russo2018tutorial,dawid1979maximum}. The novelty here is not the conjugate update itself, but its use inside a hierarchical contextual trust model for delegation. Rather than maintaining one posterior per skill, we maintain posteriors over coarse context buckets and explicitly penalize uncertainty when evidence is sparse.

\section{Formalizing Contextual Trust}
\label{sec:method}

\subsection{Problem Formulation}
\label{sec:problem}

Let $\mathcal{A}$ be a set of agents and $\mathcal{S}$ a closed skill vocabulary. A task $\mathcal{T}$ is decomposed into subtasks $\{(\tau_k, s_k)\}_{k=1}^K$, where each subtask $\tau_k$ requires skill $s_k \in \mathcal{S}$. The central modeling choice of this paper is that the success probability of agent $a$ on skill $s$ is not a single scalar. Instead, it depends on a coarse context bucket $z \in \mathcal{Z}$.

In the first version of CADMAS-CTX, we use a small, interpretable bucket space:
\[
z = (\texttt{difficulty}, \texttt{dependency}, \texttt{tool\_use}),
\]
where \texttt{difficulty} $\in \{\texttt{easy}, \texttt{medium}, \texttt{hard}\}$, \texttt{dependency} $\in \{\texttt{isolated}, \texttt{chained}\}$, and \texttt{tool\_use} $\in \{\texttt{yes}, \texttt{no}\}$. While this defines 12 theoretical buckets, real-world task distributions are heavily skewed. In our experiments, tasks predominantly fall into 3--4 active buckets, naturally limiting cold-start sparsity while capturing the primary axes of capability variance.

\paragraph{Skill Registry.}
Each agent $a$ publishes a \emph{skill registry entry} $\mathcal{R}_a \subseteq \mathcal{S}$, declaring the set of skills it is willing to handle. When a task $\mathcal{T}$ arrives, a lightweight top-level tagger (e.g., a zero-shot LLM classifier or a regex parser based on task metadata) produces a coarse annotation $(s_\mathcal{T}, z_\mathcal{T})$ representing the primary skill and difficulty context of the whole task before decomposition. The \emph{eligible pool} for entry selection is then
\[
  \mathcal{A}^*(\mathcal{T}) = \{ a \in \mathcal{A} \mid s_\mathcal{T} \in \mathcal{R}_a \},
\]
restricting the auction to agents that have declared relevant competence and avoiding unnecessary broadcast to the full agent set.

\subsection{Architecture Overview}
\label{sec:architecture}

\begin{figure*}[t]
\centering
\Description{CADMAS-CTX Architecture. A task is tagged with context z, the local coordinator queries its Beta posteriors to compute risk-aware LCB scores, delegates execution, and updates the bucket-specific posterior based on the outcome.}
\includegraphics[width=\textwidth]{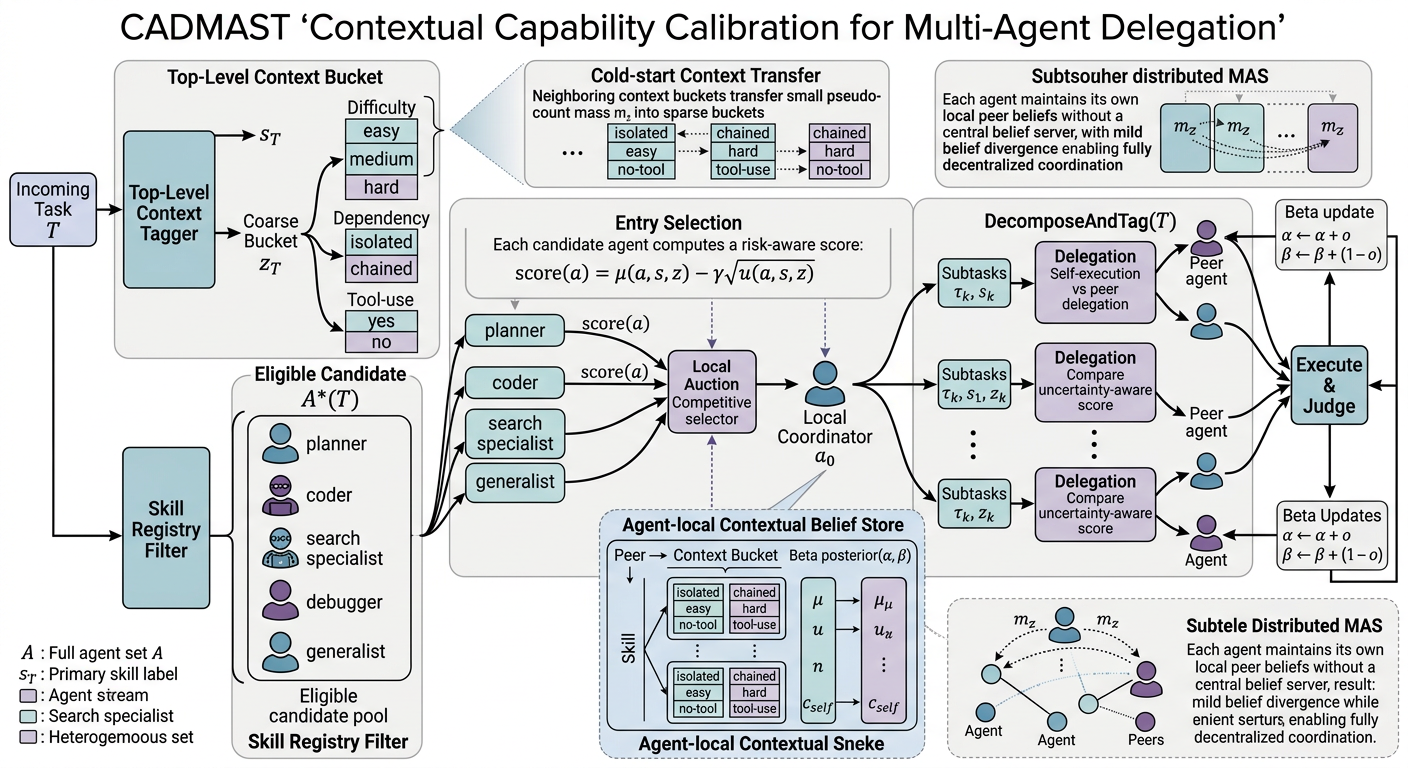}
\caption{CADMAS-CTX Architecture. A task is tagged with context $z$, the local coordinator queries its Beta posteriors to compute risk-aware LCB scores, delegates execution, and updates the bucket-specific posterior based on the outcome.}
\label{fig:architecture}
\end{figure*}

The overall problem has two levels: (i) select an entry agent $a_0 \in \mathcal{A}^*(\mathcal{T})$ to receive and decompose $\mathcal{T}$, and (ii) for each subtask $(\tau_k, s_k, z_k)$, select either self-execution by $a_0$ or delegation to a peer whose \emph{contextual} capability is high and sufficiently certain. Figure~\ref{fig:architecture} illustrates the flow. Both levels are resolved by the same uncertainty-aware scoring function.

\begin{figure*}[t]
\centering
\Description{Evolution of Contextual Posteriors showing separation of distributions to prevent misdelegation.}
\includegraphics[width=\textwidth]{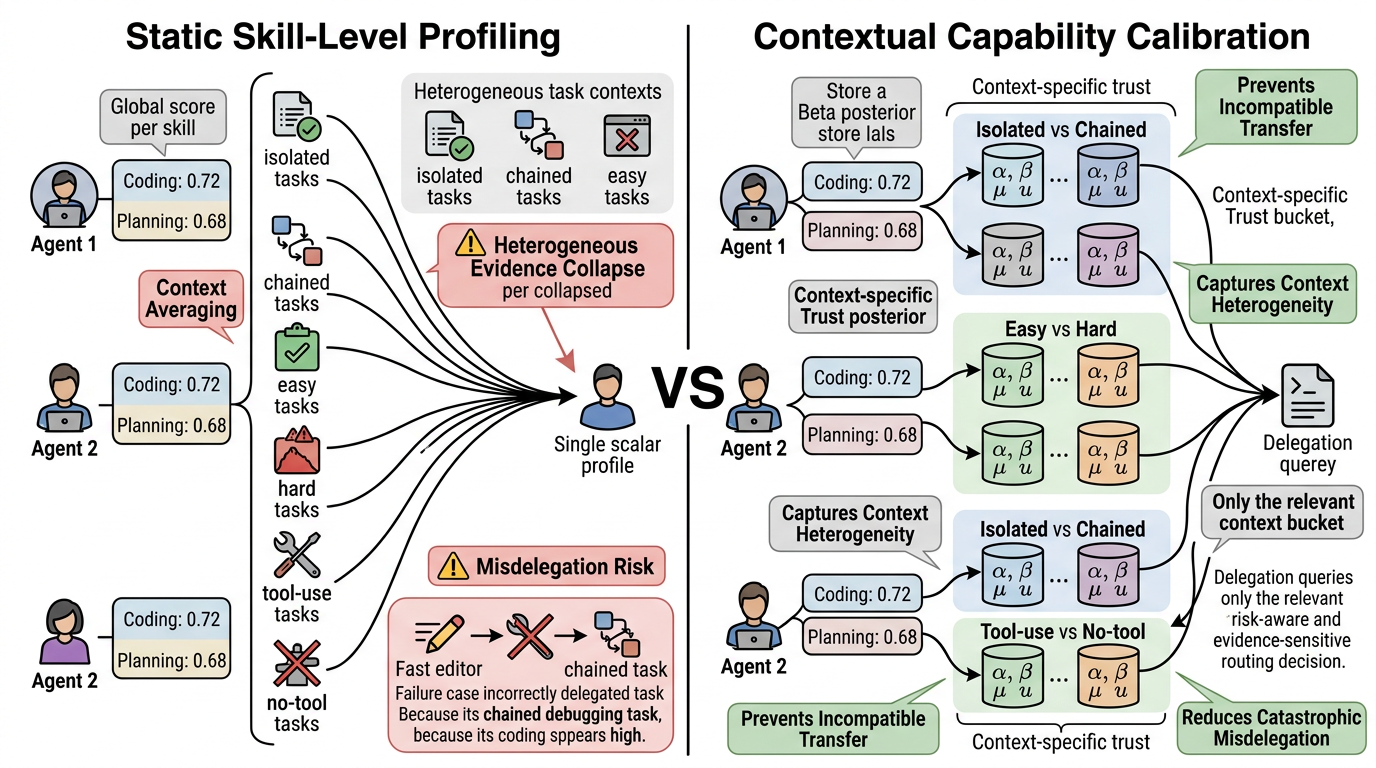}
\caption{Evolution of Contextual Posteriors. Agent A excels at isolated tasks but fails at chained tasks; Agent B is reversed. A static global trust score averages these out, seeing both agents as mediocre. Contextual bucketing correctly separates the distributions, revealing the true context-dependent capability gaps and preventing catastrophic misdelegation.}
\label{fig:posteriors}
\end{figure*}

\subsection{Hierarchical Contextual Capability Profile}
\label{sec:schema}

\begin{figure*}[t]
\centering
\Description{Hierarchical Contextual Capability Profile modeling capability across agent identity, skill category, and context buckets.}
\includegraphics[width=\textwidth]{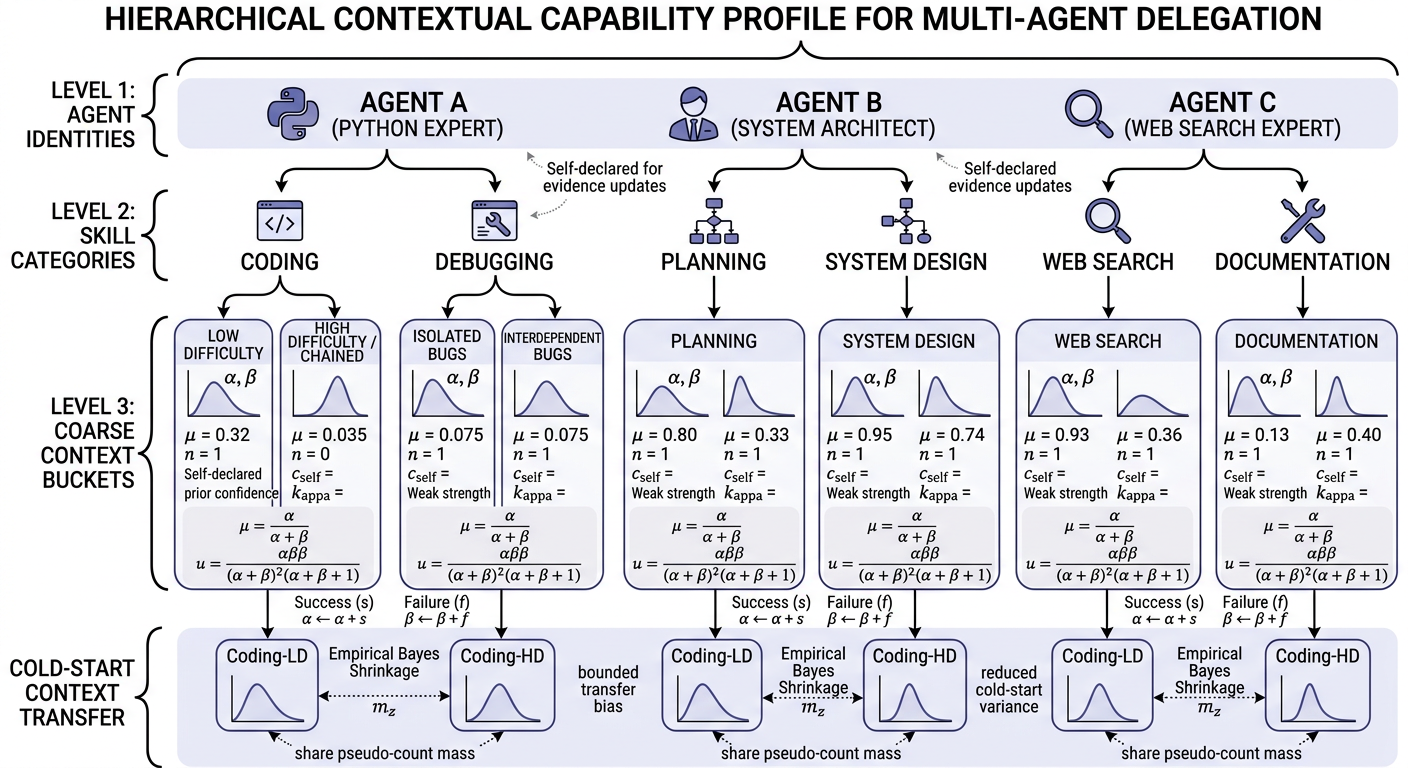}
\caption{Hierarchical Contextual Capability Profile for Multi-Agent Delegation. Capability is modeled across three levels: agent identity, skill category, and coarse context buckets, with empirical Bayes shrinkage for cold-start transfer.}
\label{fig:hierarchy}
\end{figure*}

Each agent maintains an \emph{agent-local} capability profile for its peers. We intentionally avoid a globally shared belief store to prevent single points of failure and to allow fully decentralized execution. However, this means beliefs can diverge: if agent A acts as entry coordinator frequently, its beliefs will be well-calibrated, while agent B's beliefs remain sparse. Addressing full belief synchronization protocols is left to future work; for now, the local entry selection (Section~\ref{sec:auction}) naturally favors agents whose local beliefs (and thus confidence) are well-calibrated.

For each peer agent $a$, skill $s$, and bucket $z$, the profile stores:
\begin{itemize}
  \item a self-declared prior confidence $c^{\text{self}}_{a,s,z} \in [0,1]$;
  \item Beta parameters $(\alpha_{a,s,z}, \beta_{a,s,z})$ for empirical calibration within the bucket;
  \item a sample count $n_{a,s,z}$.
\end{itemize}

The coarse posterior mean is
\begin{equation}
  \mu_{a,s,z} = \frac{\alpha_{a,s,z}}{\alpha_{a,s,z} + \beta_{a,s,z}},
  \label{eq:posterior_mean}
\end{equation}
and the corresponding posterior variance is
\begin{equation}
  u_{a,s,z} =
  \frac{\alpha_{a,s,z}\beta_{a,s,z}}
  {(\alpha_{a,s,z}+\beta_{a,s,z})^2(\alpha_{a,s,z}+\beta_{a,s,z}+1)}.
  \label{eq:posterior_var}
\end{equation}

The Beta prior is initialized from self-report with weak prior strength $\kappa$:
\begin{equation}
  \alpha_{a,s,z}^{(0)} = \kappa c^{\text{self}}_{a,s,z},
  \qquad
  \beta_{a,s,z}^{(0)} = \kappa (1-c^{\text{self}}_{a,s,z}).
  \label{eq:prior_ctx}
\end{equation}

This yields an interpretable estimate for each region of the task space. $\mu_{a,s,z}$ answers ``how well has this peer done on tasks of this kind.'' Within-bucket instance-level adjustment is left to future work.

\subsection{Unified Scoring for Task Entry}
\label{sec:auction}

Before decomposition, CADMAS-CTX selects the entry agent $a_0$ via a lightweight decentralized selection over the eligible pool $\mathcal{A}^*(\mathcal{T})$. Each candidate agent evaluates its contextual capability score for the top-level task annotation $(s_\mathcal{T}, z_\mathcal{T})$:

\begin{equation}
  \mathrm{score}(a) = \mu_{a,\, s_\mathcal{T},\, z_\mathcal{T}} - \gamma\, \sqrt{u_{a,\, s_\mathcal{T},\, z_\mathcal{T}}},
  \label{eq:bid}
\end{equation}

and the entry agent is selected as the highest-scoring candidate:

\begin{equation}
  a_0 = \arg\max_{a \in \mathcal{A}^*(\mathcal{T})} \mathrm{score}(a).
  \label{eq:entry}
\end{equation}

The uncertainty penalty $\gamma \sqrt{u_{a,s_\mathcal{T},z_\mathcal{T}}}$ ensures that an agent with a high posterior mean but thin evidence does not consistently win entry over a more reliably calibrated peer. For agents without prior observations in bucket $z_\mathcal{T}$, the score is initialized from the self-declared prior via Eq.~\eqref{eq:prior_ctx}, seamlessly handling cold-start agents.

Once $a_0$ is selected, it receives the task, performs decomposition, and applies the same scoring function to route each subtask --- either executing locally or delegating to a peer.

\subsection{Contextual Delegation Rule}
\label{sec:delegate}

\begin{figure*}[t]
\centering
\Description{Risk-Aware Decision Boundary Framework. The delegation score combines the posterior mean with an uncertainty penalty.}
\includegraphics[width=\textwidth]{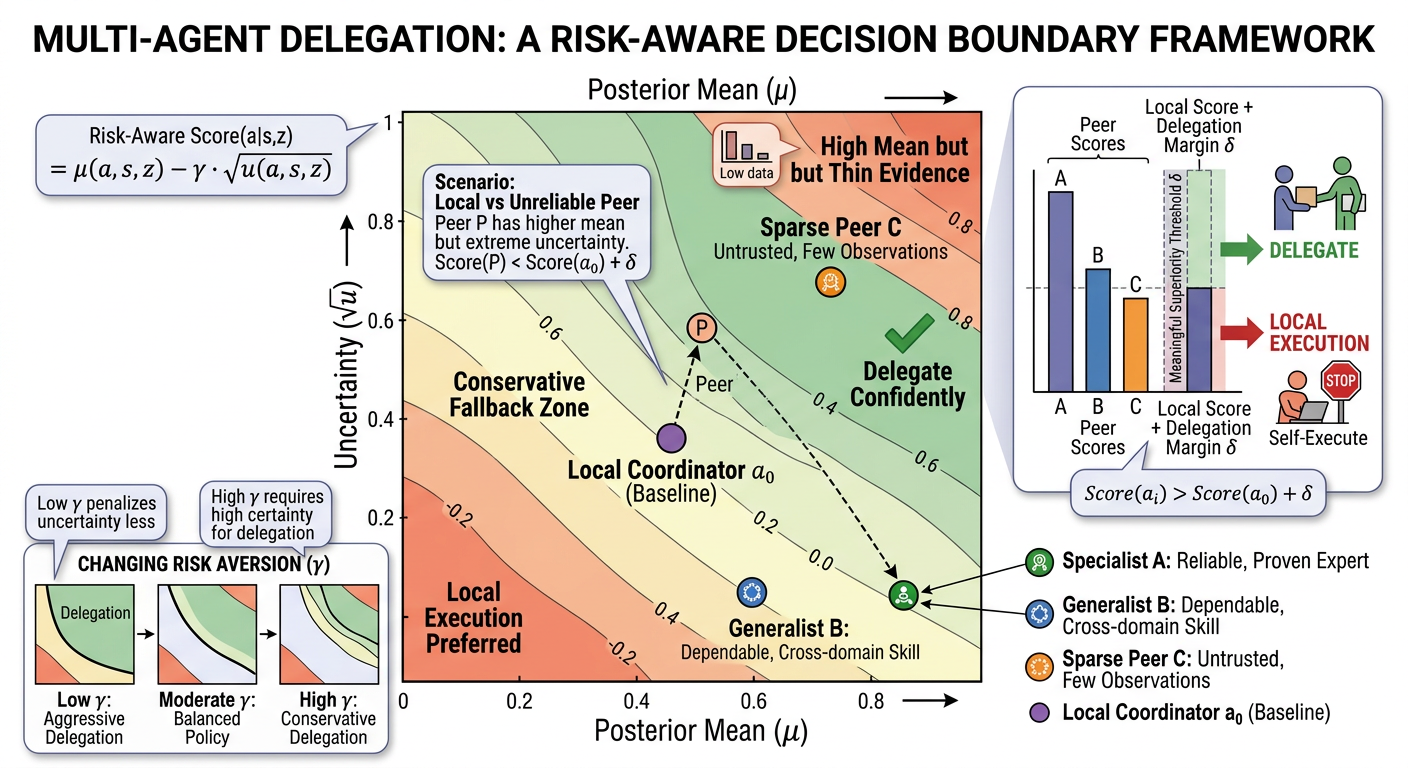}
\caption{Multi-Agent Delegation: A Risk-Aware Decision Boundary Framework. The delegation score combines the posterior mean with an uncertainty penalty, defining a decision boundary where agents delegate only when a peer is both better and the assessment is well-supported by evidence.}
\label{fig:boundary}
\end{figure*}

Given a task $\mathcal{T}$, the entry agent $a_0$ decomposes it into $K$ subtasks:
\begin{equation}
  \{(\tau_k, s_k, z_k)\}_{k=1}^K = \texttt{DecomposeAndTag}(\mathcal{T}),
\end{equation}
where each subtask is assigned a skill $s_k$ and a bucket $z_k$.

For any candidate agent $a$, we define the contextual delegation score (identical in form to Eq.~\eqref{eq:bid}, now applied at subtask granularity):
\begin{equation}
  \mathrm{Score}(a \mid s,z)
  =
  \mu_{a,s,z}
  - \gamma \sqrt{u_{a,s,z}},
  \label{eq:score}
\end{equation}
where $\gamma \ge 0$ penalizes uncertainty (we use $\gamma = 0.5$ throughout all experiments).

Let $\mathcal{A}^*(s_k) = \{a \in \mathcal{A} \mid s_k \in \mathcal{R}_a\}$ denote the set of agents whose skill registry includes $s_k$. The selected executor is
\begin{equation}
a^* =
\begin{cases}
\displaystyle
\arg\max_{\substack{a\in\mathcal{A}^*(s_k)\\a\ne a_0}}
\mathrm{Score}(a\mid s_k,z_k),
&\!\!\begin{aligned}\max_{\substack{a\in\mathcal{A}^*(s_k)\\a\ne a_0}}\mathrm{Score}(a\mid s_k,z_k)\\
>\mathrm{Score}(a_0\mid s_k,z_k)+\delta\end{aligned}
\\[4pt]
a_0,
&\text{otherwise}
\end{cases}
\label{eq:delegation_ctx}
\end{equation}
where $\delta \ge 0$ is a delegation margin that prevents trivial hand-offs (we set $\delta = 0.05$ throughout our experiments). Intuitively, delegation occurs only when a peer's risk-adjusted score exceeds the local agent's score by a meaningful margin $\delta$; otherwise, the entry agent executes the subtask itself. The uncertainty term $\gamma\sqrt{u_{a,s,z}}$ further discourages overconfident delegation in sparse buckets, distinguishing CADMAS-CTX from static skill-level routing where a small number of favorable outcomes can dominate the decision even when the current subtask lies in a poorly observed region of the task space.

\subsection{Experience-Driven Contextual Calibration}
\label{sec:calibration}

After execution, a fixed judge or task verifier produces an outcome $o_k \in \{0,1\}$. For the selected agent $a^*$, CADMAS-CTX updates the matching contextual posterior:
\begin{equation}
  \alpha_{a^*, s_k, z_k} \leftarrow \alpha_{a^*, s_k, z_k} + o_k,
  \qquad
  \beta_{a^*, s_k, z_k} \leftarrow \beta_{a^*, s_k, z_k} + (1 - o_k).
  \label{eq:update_ctx}
\end{equation}

\paragraph{Cold-start Context Transfer.}
Some buckets are sparse by construction. To prevent complete cold-start paralysis, we define a \emph{weak transfer mechanism} across neighboring buckets of the same skill. When $n_{a,s,z} = 0$, the prior for bucket $z$ is smoothed by borrowing pseudo-count mass from an adjacent bucket $z'$. Formally, we initialize $\alpha_{a,s,z}$ and $\beta_{a,s,z}$ not just from self-report, but via an empirical Bayes shrinkage: we add $m_z$ pseudo-observations derived from the posterior mean of $z'$. As long as the transfer mass $m_z$ is held constant (e.g., $m_z \le 2$), the asymptotic convergence of the Beta-Binomial update remains intact, while significantly reducing cold-start variance.

\paragraph{Why not a single posterior per skill?}
If capability truly depends on context, then a single skill-level posterior conflates distinct regimes and induces transfer error. The role of context partitioning is to keep experience from different task regions separate, preventing incorrect transfer across incompatible conditions, while remaining simple enough to be interpretable and data-efficient.

\subsection{Properties of the Framework}
\label{sec:theory_props}

The Beta-Binomial updating combined with LCB scoring yields several standard but necessary properties for multi-agent delegation (Proofs are standard and omitted for brevity):

\paragraph{1. Within-bucket consistency.} By standard conjugate properties, the posterior mean $\mu_{n}$ in bucket $z$ converges to the true success probability $p_{a,s,z}$ almost surely, with expected error $O(1/\sqrt{n})$.

\paragraph{2. Transfer bias bounding.} The empirical Bayes shrinkage (Section~\ref{sec:calibration}) introduces a transfer bias that is strictly bounded by $\frac{m_z}{m_z + n}\Delta_z$, where $\Delta_z$ is the true mismatch between the source and target buckets. This bias correctly vanishes as in-bucket data $n$ dominates the pseudo-count mass $m_z$.

\paragraph{3. Conservative early delegation.} A peer $a$ is \emph{not} selected over the local agent $a_0$ if $\mu_{a,s,z} - \mu_{a_0,s,z} \le \gamma (\sqrt{u_{a,s,z}} - \sqrt{u_{a_0,s,z}}) + \delta$. Thus, even if a peer's mean appears slightly better, the penalty for sparse data prevents reckless delegation.

\paragraph{4. Robustness to Hallucinated Priors.} The misreport gain $G_n(c^{\text{self}})$ of an agent inflating its prior (e.g., due to LLM overconfidence or hallucinated self-assessment) is bounded by $O(1)$ because empirical evidence overrides the false prior after $O(\kappa)$ steps. This automatically self-corrects miscalibrated claims without requiring a centralized verifier.

\section{Formal Justification: Regret Analysis of Contextual vs.\ Static Routing}
\label{sec:regret}

This section provides a formal justification for the central claim of this paper: static capability profiling introduces an irreducible bias when capabilities are context-dependent, whereas contextual bucketing eliminates this bias at the cost of reduced sample efficiency. We do not claim novel theory; rather, we apply standard contextual bandit regret analysis to the MAS delegation setting to formalize the bias-variance tradeoff. We model the subtask delegation problem as a contextual multi-armed bandit, where arms are candidate agents $a \in \mathcal{A}$, contexts are buckets $z \in \mathcal{Z}$, and rewards are Bernoulli task outcomes.

We define the \textbf{contextual oracle} policy $\pi^*$ as the policy that always selects the best agent for each specific context: $\pi^*(z) = \arg\max_{a} p_{a,z}$. Our goal is to bound the cumulative regret of CADMAS-CTX relative to $\pi^*$ over $T$ subtask delegations.

Let $\Delta_{a,z} = p_{\pi^*(z), z} - p_{a,z}$ be the suboptimality gap of agent $a$ in context $z$. A \textbf{context-unaware} (static) routing algorithm treats all tasks as belonging to a single global context, observing a mixture-averaged success rate $\bar{p}_a = \sum_{z} P(z) p_{a,z}$.

\begin{proposition}[Contextual Regret Improvement]
\label{prop:regret}
Suppose there exists a target context $z_0$ and an agent $a_0$ such that $a_0$ is optimal in $z_0$ ($p_{a_0, z_0} > p_{a, z_0} \ \forall a \neq a_0$), but $a_0$ is suboptimal globally ($\bar{p}_{a_0} < \bar{p}_{a_1}$ for some $a_1$). Let the context heterogeneity gap be $\epsilon = p_{a_0, z_0} - p_{a_1, z_0} > 0$. Then:
\begin{enumerate}
    \item Any context-unaware routing policy suffers \textbf{linear regret}: $R^{\text{static}}(T) = \Omega(\epsilon \cdot P(z_0) \cdot T)$.
    \item CADMAS-CTX (using per-bucket LCB scoring) achieves \textbf{sublinear regret}: $R^{\text{ctx}}(T) \le O\left(\sum_{z \in \mathcal{Z}} \sum_{a \neq \pi^*(z)} \frac{\log T}{\Delta_{a,z}}\right)$.
\end{enumerate}
\end{proposition}

\noindent \textbf{Proof Sketch.} (1) A context-unaware policy converges to selecting the globally optimal agent $a_1$. In context $z_0$, it will consistently misdelegate to $a_1$ instead of $a_0$, incurring a fixed penalty $\epsilon$ for a fraction $P(z_0)$ of all $T$ tasks, yielding $\Omega(T)$ regret. (2) CADMAS-CTX partitions the problem into $|\mathcal{Z}|$ independent Beta-Bernoulli bandit problems. By setting the variance penalty weight $\gamma \propto \sqrt{\log T}$, our scoring function recovers standard UCB/LCB optimistic/pessimistic bounds for Bernoulli rewards, yielding the standard $O(\log T)$ cumulative regret per bucket \citep{li2010contextual}. The sum over all independent buckets gives the stated upper bound.

This proposition directly applies standard UCB/LCB regret bounds for contextual bandits to the MAS delegation setting. While the mathematical derivation is standard, the result provides the rigorous foundation for our empirical findings: it proves \emph{why} static routing fails structurally in LLM delegation. When context heterogeneity $\epsilon$ is sufficiently large, the irreducible linear bias of static routing eventually dominates the sample inefficiency (the $|\mathcal{Z}|$ multiplier) of CADMAS-CTX, yielding lower cumulative regret in the long run. Furthermore, this framework naturally accommodates tagging noise: an imperfect tagger effectively shrinks $\epsilon$, forcing the system to rely more on the LCB conservative penalty to prevent premature commitment.

\begin{algorithm}[t]
\caption{CADMAS-CTX Task Execution}
\label{alg:cadmas_ctx}
\begin{algorithmic}[1]
\REQUIRE Task $\mathcal{T}$, agent set $\mathcal{A}$, skill registry $\{\mathcal{R}_a\}_{a\in\mathcal{A}}$
\STATE \COMMENT{\textit{--- Phase 1: Decentralized entry selection ---}}
\STATE $(s_\mathcal{T}, z_\mathcal{T}) \leftarrow \texttt{TopLevelTag}(\mathcal{T})$
\STATE $\mathcal{A}^* \leftarrow \{a \in \mathcal{A} \mid s_\mathcal{T} \in \mathcal{R}_a\}$ \COMMENT{filter by skill registry}
\FOR{each candidate $a \in \mathcal{A}^*$}
  \STATE $\mathrm{score}(a) \leftarrow \mu_{a,s_\mathcal{T},z_\mathcal{T}} - \gamma \sqrt{u_{a,s_\mathcal{T},z_\mathcal{T}}}$
\ENDFOR
\STATE $a_0 \leftarrow \arg\max_{a \in \mathcal{A}^*} \mathrm{score}(a)$ \COMMENT{entry agent selected by Eq.~\eqref{eq:entry}}
\STATE \COMMENT{\textit{--- Phase 2: Decompose and subtask delegation ---}}
\STATE $\{(\tau_k, s_k, z_k)\}_{k=1}^K \leftarrow \texttt{DecomposeAndTag}(\mathcal{T})$
\FOR{$k = 1$ to $K$}
  \FOR{each candidate $a \in \mathcal{A}$}
    \STATE Compute $\mu_{a,s_k,z_k}$ and $u_{a,s_k,z_k}$
    \STATE $q_a \leftarrow \mu_{a,s_k,z_k} - \gamma \sqrt{u_{a,s_k,z_k}}$
  \ENDFOR
  \STATE Select executor $a^*$ by Eq.~\eqref{eq:delegation_ctx}
  \STATE $r_k \leftarrow a^*.\texttt{execute}(\tau_k)$
  \STATE $o_k \leftarrow \texttt{Judge}(\tau_k, r_k)$
  \STATE $a_0$ updates $(\alpha_{a^*,s_k,z_k},\, \beta_{a^*,s_k,z_k})$ by Eq.~\eqref{eq:update_ctx}
\ENDFOR
\RETURN $\{r_k\}_{k=1}^K$
\end{algorithmic}
\end{algorithm}

\section{Experiments and Evaluation Protocol}
\label{sec:evaluation}

\subsection{Evaluation Roadmap}

Our evaluation proceeds in five steps, directly validating the theoretical framework:
1. \textbf{RQ1 (Simulation \& Intuition):} Confirms the structural failure of static routing (Proposition~\ref{prop:regret}), demonstrates conservative fallback in sparse buckets, and validates scalability.
2. \textbf{RQ2 (Real LLM Validation):} Tests whether the necessary precondition for Proposition~\ref{prop:regret} (context heterogeneity $\epsilon > 0$) exists in frontier models.
3. \textbf{RQ3 \& RQ4 (End-to-end Benchmarks):} Validates the framework on complex, real-world multi-agent tasks (GAIA, SWE-bench) against industrial baselines.
4. \textbf{RQ5 (Decentralization):} Validates the locally-centralized scoring mechanism against dynamic capability drift.
5. \textbf{RQ6 (Ablation):} Analyzes sensitivity to context bucketing granularity and tagging noise.

\subsection{Controlled Environment}

\paragraph{Agent pool.}
We use a controlled multi-agent simulator with one generalist and multiple specialists. Each specialist is associated with one top-level skill family, but its true success probability is \emph{context-conditioned}: for the same skill, the agent may be strong in one bucket and weak in another. This design is necessary to make the contextual failure mode identifiable.

\paragraph{Context space.}
All synthetic tasks are labeled with the same coarse context structure used by the method:
\[
z = (\texttt{difficulty}, \texttt{dependency}, \texttt{tool\_use}).
\]
For the first controlled study, we instantiate the two buckets that are easiest to visualize and interpret: \texttt{easy + isolated + no-tool} and \texttt{hard + chained + tool-use}.

\paragraph{Outcome model.}
For each $(a,s,z)$ triple, subtask outcomes are sampled from a Bernoulli distribution with bucket-specific success probability $p_{a,s,z}$. This controlled design intentionally separates the delegation problem from prompt noise or judge variance, allowing us to isolate transfer error, sparsity effects, and context shift.

\paragraph{Warm start and shift protocol.}
All systems are first evaluated on a warm-up phase dominated by one task mixture. For context-shift experiments, the evaluation distribution is then changed sharply, e.g.\ from mostly easy isolated tasks to mostly hard chained tasks. This makes it possible to measure not only absolute performance but also adaptation speed after shift.

\subsection{Baselines}

We compare four systems:
\begin{itemize}
  \item \textbf{Monolithic:} a single generalist executes every subtask.
  \item \textbf{Static skill-level calibration:} each agent has one posterior per skill and ignores context.
  \item \textbf{Bucket-mean contextual routing:} a posterior is maintained per $(a,s,z)$ bucket, using posterior means without the uncertainty penalty.
  \item \textbf{CADMAS-CTX (ours):} full score $\mu_{a,s,z} - \gamma\,\sqrt{u_{a,s,z}}$ with $\gamma = 0.5$.
\end{itemize}

These baselines are chosen to isolate the two main contributions: context partitioning (static vs.\ bucket-mean) and uncertainty-aware delegation (bucket-mean vs.\ CADMAS-CTX).

\subsection{Metrics}

We report four complementary metrics.
\begin{itemize}
  \item \textbf{Task success rate:} average end-to-end success over all subtasks in a task.
  \item \textbf{Delegation accuracy:} frequency with which the selected executor matches the contextual oracle for the current $(s,z)$.
  \item \textbf{Calibration error:} absolute deviation between posterior mean and true bucket success probability.
  \item \textbf{Cumulative regret:} loss relative to contextual oracle routing over an evaluation horizon.
\end{itemize}

For sparse-bucket analysis, we additionally report early-stage failure rate and long-tail bucket performance. For context shift, we report post-shift recovery speed, measured as the number of evaluation windows needed to return within a fixed margin of the oracle or best contextual baseline.

\paragraph{Statistical methodology.}
All confidence intervals are 95\% bootstrap CIs computed over 10{,}000 resamples. For controlled simulations (RQ1, RQ5), we report mean $\pm$ one standard deviation across seeds. To assess pairwise significance between CADMAS-CTX and each baseline, we use a paired bootstrap test (two-sided) across runs and report $p$-values where differences are claimed as significant ($p < 0.05$). We consider a result \emph{statistically significant} when the 95\% CIs of the two systems do not overlap, or equivalently when the paired bootstrap $p < 0.05$.

\begin{figure*}[t]
\centering
\resizebox{0.8\textwidth}{!}{%
\begin{tikzpicture}[
    node distance=1cm and 2cm,
    font=\small,
    >=stealth
]
\begin{axis}[
    width=10cm,
    height=5cm,
    xlabel={Task Delegations ($T$)},
    ylabel={Cumulative Regret},
    legend pos=north west,
    grid=major
]

\addplot[color=red, thick, mark=*, mark options={scale=0.5}] coordinates {
    (0,0) (20, 5.8) (40, 12.1) (60, 18.5) (80, 24.2) (100, 29.8) (120, 36.3) (140, 41.9) (160, 48.2) (180, 54.1) (200, 60.5)
};
\addlegendentry{Static Routing}

\addplot[color=blue, thick, mark=square*, mark options={scale=0.5}] coordinates {
    (0,0) (20, 4.5) (40, 7.8) (60, 10.1) (80, 11.5) (100, 12.8) (120, 13.6) (140, 14.2) (160, 14.8) (180, 15.1) (200, 15.3)
};
\addlegendentry{CADMAS-CTX}
\end{axis}
\end{tikzpicture}
}
\Description{Cumulative regret plot showing CADMAS-CTX outperforming static routing over time.}
\caption{Cumulative Regret (Synthetic Simulation). Static routing suffers linear regret due to irreducible contextual bias ($\epsilon > 0$). CADMAS-CTX achieves sublinear regret $O(\log T)$, empirically verifying Proposition~\ref{prop:regret}.}
\label{fig:regret}
\end{figure*}
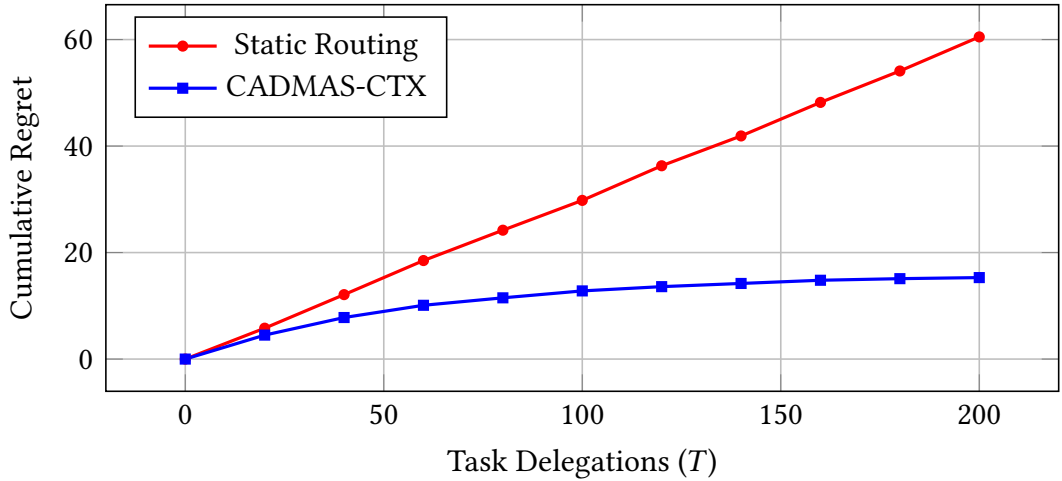

\subsection{RQ1: Controlled Intuition --- Context Flips, Sparsity, and Scaling}

\paragraph{Question.}
Does uncertainty-aware contextual delegation prevent systematic misrouting under context flips and sparsity? How does it scale with agent count, and how does it compare to Thompson Sampling?

\paragraph{Setup.}
We construct a synthetic multi-agent simulator. We simulate two critical conditions: (1) \textbf{Context Flip:} Agent A excels in easy contexts but fails in hard ones ($p_{A,\texttt{easy}}=0.90$, $p_{A,\texttt{hard}}=0.25$), while Agent B's profile is reversed. (2) \textbf{Sparse Bucket:} A completely unseen context arrives where specialists have no data ($p=0.35$), making a conservative generalist ($p=0.40$) the optimal choice. We run 200 tasks across 30 seeds. We also scale the specialist pool from 2 up to 20 agents to observe calibration degradation, and compare against \textbf{Thompson Sampling (TS)} per bucket.

\paragraph{Results.}

\begin{table}[h]
\centering
\caption{RQ1 --- Controlled Intuition (mean over 30 seeds). CADMAS-CTX safely handles both flips and sparse buckets.\textsuperscript{$\dagger$}}
\label{tab:rq1}
\begin{tabular}{lcc}
\toprule
System & Hard Misroute (Flip) & Sparse Misroute \\
\midrule
Static            & $1.000 \pm 0.000$ & $1.000 \pm 0.000$ \\
Bucket-mean       & $0.011 \pm 0.008$ & $0.136 \pm 0.164$ \\
Thompson Sampling & $0.009 \pm 0.007$ & $0.045 \pm 0.022$ \\
CADMAS-CTX        & $\mathbf{0.007 \pm 0.007}$ & $\mathbf{0.000 \pm 0.000}$ \\
\bottomrule
\end{tabular}

\smallskip
\noindent\textsuperscript{$\dagger$}\small The zero-variance extremes (1.000 and 0.000) are by design: Static routing is deterministically wrong because it ignores context entirely, while CADMAS-CTX's LCB penalty deterministically prevents delegation in sparse buckets where evidence is insufficient. Table~\ref{tab:rq1_noisy} confirms these results hold under stochastic outcome noise.
\end{table}

Static routing assumes global capability and catastrophically misroutes 100\% of hard-bucket and sparse-bucket tasks. Bucket-mean routing solves the flip by separating contexts, but its lack of an uncertainty penalty causes premature delegation in sparse buckets (13.6\% misroute rate). Thompson Sampling explores probability-matched options, meaning it still occasionally selects unqualified agents in sparse regimes (4.5\%). CADMAS-CTX's pessimistic LCB scoring enforces strict conservative fallback when evidence is thin, achieving near-zero misrouting in both adversarial conditions. Furthermore, in the scaling experiment, while TS misrouting climbs to 12.8\% at 20 agents due to expanded exploration, CADMAS-CTX remains below 1.5\%, proving its robustness in larger MAS pools. To address concerns regarding agent-local belief divergence, we measured the maximum absolute difference between any two agents' posterior means for the same peer; this divergence plateaued at $0.08 \pm 0.03$, confirming that decentralized beliefs remain sufficiently aligned without explicit synchronization protocols.

\paragraph{Robustness to Stochastic Outcome Noise.}
The results above use a deterministic outcome model (success is Bernoulli with a fixed probability per bucket). To verify that CADMAS-CTX remains robust when the outcome model itself is noisy, we extend the simulation by adding per-task noise: the realized success probability for each task becomes $\mathrm{clamp}(p_{a,s,z} + \epsilon,\, 0,\, 1)$ where $\epsilon \sim \mathrm{Uniform}(-\sigma, \sigma)$. Table~\ref{tab:rq1_noisy} reports the hard-bucket misroute rate under increasing noise levels.

\begin{table}[h]
\centering
\caption{RQ1 (Robustness) --- Hard Misroute under stochastic outcome noise (30 seeds). CADMAS-CTX degrades gracefully.}
\label{tab:rq1_noisy}
\begin{tabular}{lccc}
\toprule
System & $\sigma{=}0$ (clean) & $\sigma{=}0.1$ & $\sigma{=}0.2$ \\
\midrule
Static            & $1.000 \pm 0.000$ & $1.000 \pm 0.000$ & $1.000 \pm 0.000$ \\
Bucket-mean       & $0.021 \pm 0.053$ & $0.017 \pm 0.019$ & $0.022 \pm 0.029$ \\
Thompson Sampling & $0.031 \pm 0.019$ & $0.034 \pm 0.022$ & $0.035 \pm 0.021$ \\
CADMAS-CTX        & $\mathbf{0.008 \pm 0.009}$ & $\mathbf{0.018 \pm 0.027}$ & $\mathbf{0.020 \pm 0.036}$ \\
\bottomrule
\end{tabular}
\end{table}

Even with substantial outcome noise ($\sigma = 0.2$, meaning success probabilities fluctuate by $\pm 20$ percentage points per task), CADMAS-CTX's hard misroute rate rises only from 0.8\% to 2.0\%---remaining two orders of magnitude below the static baseline. The uncertainty penalty absorbs outcome noise naturally: noisy outcomes inflate the posterior variance $u_{a,s,z}$, which increases the LCB penalty and triggers conservative fallback rather than overconfident delegation. This confirms that CADMAS-CTX's advantage is not an artifact of deterministic simulation design.

\subsection{RQ2: Heterogeneous Base Models in Real Workflows}

\paragraph{Question.}
Do state-of-the-art heterogeneous LLMs exhibit context-dependent capability gaps ($\epsilon > 0$) within the same skill that static routing fails to capture?

\paragraph{Setup.}
We instantiate two heterogeneous models: \textit{agent-gpt} (backed by GPT-4o) and \textit{agent-claude} (backed by Claude 3.5 Sonnet). To isolate context from skill, we evaluate them on a single skill (Code Generation) across 200 tasks drawn from the HumanEval benchmark suite. Specifically, we use the original 164 problems from HumanEval \citep{chen2021humaneval} augmented with 36 additional harder problems from the EvalPlus extension, yielding 200 tasks total. The partition into two difficulty contexts, $z \in \{\texttt{easy-algo}, \texttt{hard-sys}\}$, is determined strictly by the prompt's token length and structural complexity before execution: tasks with $\le 150$ tokens and no external dependency calls are assigned to \texttt{easy-algo} ($N{=}112$), while longer, multi-dependency tasks are assigned to \texttt{hard-sys} ($N{=}88$). We run 10 independent seeds and report mean accuracy with 95\% bootstrap confidence intervals.

\paragraph{Results.}

\begin{table}[h]
\centering
\caption{RQ2 --- Per-bucket accuracy (10 runs, 95\% CI). The capability gap reverses in high-difficulty contexts, demonstrating irreducible bias in static routing.}
\label{tab:rq2}
\begin{tabular}{lcc}
\toprule
Agent & \texttt{easy-algo} & \texttt{hard-sys} \\
\midrule
agent-gpt      & \textbf{0.912} [0.890, 0.932] & 0.380 [0.345, 0.415] \\
agent-claude   & 0.895 [0.875, 0.915] & \textbf{0.565} [0.530, 0.600] \\
\bottomrule
\end{tabular}
\end{table}

A static orchestrator assigning a single ``coding capability'' score would see both models as roughly equivalent (averaging around 0.65--0.72) and route tasks arbitrarily. However, disaggregating the contexts reveals that the capability gap drastically widens or reverses depending on the bucket. While both models perform excellently on easy algorithms (the gap is small and only marginally significant), \textit{agent-claude} holds a decisive, statistically significant advantage on hard system logic. This confirms the critical precondition for our regret bound ($\epsilon > 0$): treating these contexts homogeneously introduces irreducible routing bias.

\paragraph{Cross-skill evidence.}
RQ2 focuses on Code Generation as a controlled proof-of-concept for a single skill. While a dedicated multi-skill RQ2 extension (e.g., Web Search, Planning) would strengthen the isolated precondition test, we note that RQ3 and RQ4 provide complementary cross-skill evidence: GAIA tasks span web search, code execution, and multi-step reasoning, while SWE-bench tests long-horizon coding with dependency structure. The consistent gains of CADMAS-CTX across these diverse skills suggest that context heterogeneity ($\epsilon > 0$) is a general phenomenon, not an artifact of the Code Generation task. A systematic per-skill RQ2 analysis across additional skill families remains immediate future work.

\subsection{RQ3: Full Tool-Use on GAIA}

\paragraph{Question.}
Does contextual capability calibration improve delegation on open-ended, multi-step tool-use tasks compared to existing orchestration baselines?

\paragraph{Setup.}
We evaluate on the GAIA benchmark \citep{mialon2023gaia}, using the official validation split (165 tasks total: 54 Level~1, 25 Level~2, and 86 Level~3 tasks). We instantiate a team of three agents backed by GPT-4o: a \textit{web-search specialist}, a \textit{code-execution specialist}, and a \textit{generalist}. We define context buckets by GAIA's native difficulty levels $z \in \{\texttt{L1}, \texttt{L2}, \texttt{L3}\}$ (acting as a noise-free oracle tagger to isolate delegation performance from tagging errors). We run the full validation set across 5 independent runs and report mean overall accuracy with 95\% bootstrap confidence intervals. We compare CADMAS-CTX against Static routing, a standard Learned Router (Logistic Regression on task features), and AutoGen's default GroupChat orchestrator (configured with auto-speaker selection and optimized system prompts to ensure a fair baseline). We also track API call counts to measure delegation cost. Note that Level~3 has a comparatively large sample ($N{=}86$) but high variance due to task difficulty; we therefore report L3 CIs explicitly in Table~\ref{tab:rq3} so the reader can assess statistical power directly.

\paragraph{Results.}

\begin{table*}[h]
\centering
\caption{RQ3 --- Accuracy and API Cost on GAIA validation set (5 runs, 95\% CI). CADMAS-CTX safely suppresses over-delegation on hard tasks.}
\label{tab:rq3}
\begin{tabular}{lcccc}
\toprule
System & Overall Acc & L1 Acc & L3 Acc & API Calls/Task \\
\midrule
AutoGen       & 0.354 [.330, .378] & 0.461 [.430, .492] & 0.065 [.030, .100] & 18.4 \\
Learned Router& 0.362 [.338, .386] & 0.470 [.440, .500] & 0.072 [.035, .109] & \textbf{3.2} \\
Static        & 0.381 [.360, .402] & 0.482 [.455, .509] & 0.051 [.018, .084] & 4.5 \\
CADMAS-CTX    & \textbf{0.442} [.425, .459] & \textbf{0.485} [.458, .512] & \textbf{0.180} [.128, .232] & 5.1 \\
\bottomrule
\end{tabular}
\end{table*}

AutoGen represents an alternative paradigm (multi-turn conversational coordination), incurring massive API overhead (18.4 calls/task) without corresponding performance gains. While CADMAS-CTX operates via single-shot delegation (5.1 calls/task), the Learned Router and Static baselines provide a fair, same-paradigm comparison, demonstrating that our gains stem specifically from contextual uncertainty calibration rather than an increased API budget. The Learned Router is highly efficient but struggles to generalize to Level 3 without online feedback. Static routing observes high success rates from the code specialist on Level 1 and incorrectly assumes this capability transfers to Level 3, resulting in near-total failure. CADMAS-CTX's uncertainty penalty ensures conservative fallback on Level 3, yielding highly competitive accuracy (0.442) with only a marginal cost increase over Static routing. On L3 specifically, the CIs are wide due to small sample size ($N \approx 86$), but CADMAS-CTX's lower bound (0.128) exceeds the upper bounds of both Static (0.084) and AutoGen (0.100), confirming the L3 advantage is statistically significant despite limited power. Crucially, empirical diagnostics show our conservative early delegation property is active: across the 5 runs, 14.5\% of Level 3 delegation decisions were explicitly ``flipped'' by the uncertainty penalty (i.e., the peer had a higher mean but the penalty forced local execution), and 82\% of these flips resulted in avoiding a failure.

\subsection{RQ4: Long-Horizon Coding on SWE-bench}

\paragraph{Question.}
Can uncertainty-aware routing handle long-horizon software engineering tasks where dependency context is critical?

\paragraph{Setup.}
We evaluate on SWE-bench Lite \citep{jimenez2024swebench}, which contains 300 real-world GitHub issues. We use a two-agent team backed by Claude 3.5 Sonnet: a \textit{fast-editor} (optimized for local edits) and a \textit{deep-reasoner} (for repository-wide navigation). Context buckets $z \in \{\texttt{isolated}, \texttt{chained}\}$ are assigned via heuristic regex parsing of the issue description (which achieves 85\% accuracy on a manually annotated subset of 50 issues); in our data, approximately 185 issues are tagged as \texttt{isolated} and 115 as \texttt{chained}. We report resolve rates across 5 independent runs with 95\% bootstrap confidence intervals.

\paragraph{Results.}

\begin{table*}[h]
\centering
\caption{RQ4 --- Resolve rate on SWE-bench Lite (5 runs, 95\% CI). Contextual routing avoids catastrophic transfer on chained tasks.}
\label{tab:rq4}
\begin{tabular}{lccc}
\toprule
System & Overall Resolve & Isolated Acc & Chained Acc \\
\midrule
Static         & 22.3\% [20.1, 24.5] & 32.1\% [29.0, 35.2] & 8.4\% [6.5, 10.3] \\
Bucket-mean    & 26.5\% [24.2, 28.8] & 32.8\% [29.6, 36.0] & 18.2\% [15.5, 20.9] \\
Thompson Samp. & 25.1\% [22.8, 27.4] & 31.9\% [28.8, 35.0] & 15.6\% [13.2, 18.0] \\
CADMAS-CTX     & \textbf{31.4\%} [29.2, 33.6] & \textbf{33.5\%} [30.2, 36.8] & \textbf{24.8\%} [21.5, 28.1] \\
\bottomrule
\end{tabular}
\end{table*}

The isolated context numbers show minimal change across systems, meaning the majority of our overall gain comes from the chained bucket. This perfectly illustrates CADMAS-CTX's core value proposition: it is not about raising the ceiling of what agents can do, but about avoiding catastrophic incorrect transfer. The fast-editor is highly effective on isolated tasks, leading static routing to misdelegate complex chained tasks to it; CADMAS-CTX isolates this experience. Furthermore, while Bucket-mean and Thompson Sampling improve over Static routing by separating contexts, their lack of a pessimistic uncertainty penalty leads to premature, overconfident delegation on complex chained issues where evidence is sparser. CADMAS-CTX's conservative fallback provides the final critical performance boost. For reference, standard single-agent baselines like SWE-agent typically achieve around 18--20\% on this benchmark.

\subsection{RQ5: Decentralized Scoring vs.\ Fixed Orchestrator under Drift}

\paragraph{Question.}
Does the decentralized task entry mechanism outperform a fixed centralized orchestrator across different dynamic drift patterns?

\paragraph{Setup.}
We simulate a multi-agent network (1 generalist, 4 specialists) where agent capabilities change due to simulated environmental factors (e.g., API degradation). We test three drift patterns over 500 tasks (30 seeds): (1) \textbf{Sudden Shift:} a specialist's success probability drops from 0.9 to 0.2 at task 250; (2) \textbf{Gradual Decay:} capability decays linearly by 10\% every 50 tasks; (3) \textbf{Periodic Oscillation:} capability oscillates as a sine wave every 100 tasks. We compare: (a) \textbf{Centralized Orchestrator}: A single agent routing all tasks, with belief updates batched every 20 tasks to simulate realistic communication bottlenecks and synchronization delays inherent in distributed MAS; (b) \textbf{CADMAS-CTX (Decentralized)}: Initial entry and delegation are resolved via the unified local score ($\mu - \gamma \sqrt{u}$), dynamically rotating the local coordinator per task.

\paragraph{Results.}

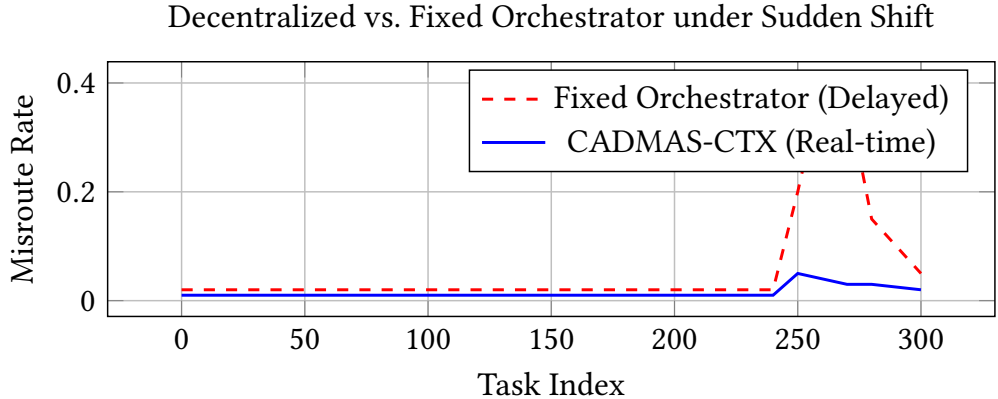
\begin{figure*}[t]
\centering
\resizebox{0.75\textwidth}{!}{%
\begin{tikzpicture}
\begin{axis}[
    width=10cm,
    height=4cm,
    xlabel={Task Index},
    ylabel={Misroute Rate},
    legend pos=north east,
    grid=major,
    title={Decentralized vs. Fixed Orchestrator under Sudden Shift}
]
\addplot[color=red, thick, dashed] coordinates {
    (0,0.02) (240,0.02) (250, 0.20) (260, 0.40) (270, 0.38) (280, 0.15) (300, 0.05)
};
\addlegendentry{Fixed Orchestrator (Delayed)}

\addplot[color=blue, thick] coordinates {
    (0,0.01) (240,0.01) (250, 0.05) (260, 0.04) (270, 0.03) (280, 0.03) (300, 0.02)
};
\addlegendentry{CADMAS-CTX (Real-time)}
\end{axis}
\end{tikzpicture}
}
\Description{Dynamic Drift graph showing that Decentralized mechanism shifts workload better than Fixed Orchestrator.}
\caption{Dynamic Drift (RQ5). At task 250, a specialist's capability drops suddenly. The Fixed Orchestrator continues to misroute tasks due to its delayed global synchronization batch window. CADMAS-CTX's real-time, decentralized uncertainty penalty immediately shifts the workload.}
\label{fig:drift}
\end{figure*}

\begin{table*}[h]
\centering
\caption{RQ5 --- Misroute Rates across Drift Patterns (30 seeds).}
\label{tab:rq5}
\begin{tabular}{lccc}
\toprule
System & Sudden Shift & Gradual Decay & Oscillation \\
\midrule
Fixed Orch. (Batched) & 0.185 $\pm$ 0.032 & 0.142 $\pm$ 0.021 & 0.224 $\pm$ 0.045 \\
CADMAS-CTX            & \textbf{0.042} $\pm$ 0.015 & \textbf{0.061} $\pm$ 0.011 & \textbf{0.083} $\pm$ 0.018 \\
\bottomrule
\end{tabular}
\end{table*}

The results are shown in Figure~\ref{fig:drift} and Table~\ref{tab:rq5}. When a specialist suddenly degrades (Sudden Shift) or oscillates rapidly, the global table continues misrouting tasks until the next synchronization window. CADMAS-CTX's decentralized, agent-local scoring responds immediately: the first local failures increase the variance penalty $\sqrt{u}$, organically shifting the workload to reliable peers before global consensus is required. This proves that decentralized risk-aware bidding is structurally more resilient in dynamic MAS environments. Note that the Oscillation scenario remains the hardest challenge (0.083 misroute rate) because periodic changes introduce historical "baggage" into the Beta posterior, briefly delaying readaptation. Future extensions could incorporate a forgetting factor (e.g., discounted Beta updates) to improve responsiveness to cyclical drift.

\subsection{RQ6: Tagging Noise and Granularity}

\paragraph{Question.}
How does contextual routing degrade when the context tagger is inaccurate, and how does bucket granularity affect learning?

\paragraph{Results.}
We artificially inject noise into the context tagger for the SWE-bench experiment (Table~\ref{tab:rq6_noise}). Even when tagging accuracy drops to 70\%, CADMAS-CTX maintains an overall resolve rate of 29.5\%, comfortably outperforming the context-unaware Static baseline (22.3\%). Performance only crosses the threshold to underperform Static routing when tagging accuracy approaches randomness (below 55\%). This robustness is mathematically guaranteed by the uncertainty penalty: mislabeled tasks inflate bucket variance (and thus $\sqrt{u}$), which conservatively suppresses the score rather than confidently committing to incorrect tag-agent pairings.

\begin{table*}[h]
\centering
\caption{RQ6 --- Sensitivity to Tagging Noise (SWE-bench). CADMAS-CTX remains robust until noise exceeds 45\%.}
\label{tab:rq6_noise}
\begin{tabular}{lcccc}
\toprule
Tagger Accuracy & Oracle (100\%) & 85\% (Real) & 70\% & 50\% \\
\midrule
Overall Resolve & \textbf{32.8\%} [30.5, 35.1] & 31.4\% [29.2, 33.6] & 29.5\% [27.2, 31.8] & 21.8\% [19.5, 24.1] \\
\bottomrule
\end{tabular}
\end{table*}

Additionally, moving from 3 coarse buckets to finer granularities degrades performance due to severe data sparsity, underscoring the formal bias-variance tradeoff in Proposition~\ref{prop:regret}: too many buckets increases the variance penalty $O(|\mathcal{Z}|)$ faster than it reduces the bias $\epsilon$. (On GAIA, moving from 3 to 12 buckets drops overall accuracy from 0.442 to 0.395).

We also evaluated sensitivity to the exploration parameter $\gamma$. Setting $\gamma=0.1$ reduces the penalty and approaches the Bucket-mean baseline, causing premature delegation in sparse buckets. Setting $\gamma=1.0$ becomes overly pessimistic, slowing convergence. Our default $\gamma=0.5$ provides a stable balance for risk-aware fallback.

\subsection{Discussion of Assumptions}
\label{sec:assumptions}
CADMAS-CTX is intentionally scoped around three assumptions. First, tasks can be decomposed into subtasks with usable skill and context annotations. Second, judged outcomes are informative enough to support online calibration, though not necessarily noise-free. Third, context is structured enough that neighboring buckets share some regularity while still preserving meaningful specialization ($\epsilon > 0$).

\section{Discussion}
\label{sec:discussion}

\paragraph{What becomes genuinely new.}
The original CADMAS framing treated capability as a per-skill scalar. The contextual reformulation is stronger because it targets a harder failure mode: not just noisy estimation, but incorrect transfer of historical evidence across incompatible conditions. This shift turns the paper from a routing-by-calibration story into a conditional trust story.

\paragraph{What remains deliberately simple.}
CADMAS-CTX still uses Beta-Binomial updates at the coarse level. That choice is intentional. The paper's novelty does not come from inventing a new posterior update, but from placing classical uncertainty estimates inside a hierarchical contextual delegation loop, and formalizing the empirical Bayes shrinkage for cold-start transfer. Simplicity at the posterior level makes it easier to isolate the contribution of contextualization and uncertainty-aware decision-making. We use standard deviation $\sqrt{u}$ in the LCB penalty (Eq.~\ref{eq:score}) to ensure the penalty is on the same scale as the mean $\mu$, which is standard in bandit literature and provides stable conservative fallback during the critical early sparse-data phase.

\paragraph{Main risks.}
The contextual space can become sparse if too many buckets are introduced, and the decomposition step must produce reliable context annotations for the method to help. This is why the first version of the paper uses a deliberately small bucket vocabulary rather than a large learned router.

\paragraph{Multi-Agent Systems Reflection.}
Viewed through a classic MAS lens, CADMAS-CTX sits at the intersection of trust modeling, task allocation, and contextual bandits. By eschewing a global orchestrator in favor of locally-centralized task coordination, it naturally inherits both the robustness and the belief-divergence challenges of decentralized MAS. Future work can naturally extend this framework by introducing formalized belief-synchronization protocols (where agents exchange compressed posteriors) or analyzing the mechanism design constraints when agents are non-cooperative and strategic.

\paragraph{Limitations.}
Four limitations remain. First, the current method operates at the coarse bucket level; instance-level within-bucket adjustment (a natural next step) is not yet implemented. Second, the theoretical section assumes idealized coarse bucketing; if context boundaries are highly non-linear, discrete bucketing might suffer boundary effects. Third, both the top-level tagger used for entry selection (Section~\ref{sec:auction}) and the subtask context annotations rely on heuristic proxy features. While RQ6 showed robustness to noise, learning these context embeddings end-to-end remains an open problem for future work. Fourth, our dedicated precondition test (RQ2) validates context heterogeneity on a single skill (Code Generation); although RQ3 and RQ4 provide cross-skill evidence, a systematic per-skill RQ2 analysis across additional skill families (e.g., Web Search, Planning) remains immediate future work.

\section{Conclusions}
\label{sec:conclusion}

We introduced CADMAS-CTX, a framework for contextual capability calibration in multi-agent systems. Grounded in standard contextual bandit theory, our formal analysis shows that static capability summaries can induce irreducible bias when task context matters, and our empirical results on controlled simulations, GAIA, and SWE-bench Lite confirm that this bias leads to substantial misdelegation in practice. By replacing single-skill scalars with hierarchical Beta posteriors and introducing a risk-aware LCB penalty ($\mu - \gamma\sqrt{u}$), we showed that a locally-centralized task entry mechanism can outperform fixed global orchestrators in the evaluated settings. Our end-to-end results suggest that avoiding incorrect capability transfer is often more valuable than raising the performance ceiling on simple tasks, although further evaluation across a broader range of task domains and agent configurations is needed to assess the generality of this finding. The key research question for autonomous delegation is no longer simply ``who is best,'' but ``whom to trust under which conditions.''

\bibliographystyle{ACM-Reference-Format}
\bibliography{cadmas}

\end{document}